\documentclass[letterpaper]{article} % DO NOT CHANGE THIS
\usepackage[preprint]{aaai2027}  % DO NOT CHANGE THIS
\usepackage[hyphens]{url}  % DO NOT CHANGE THIS
\usepackage{graphicx} % DO NOT CHANGE THIS
\usepackage{natbib}  % DO NOT CHANGE THIS AND DO NOT ADD ANY OPTIONS TO IT
\usepackage{caption} % DO NOT CHANGE THIS AND DO NOT ADD ANY OPTIONS TO IT
\usepackage{algorithm}
\usepackage{algorithmic}
\usepackage{amsmath}
\usepackage{amssymb}
\usepackage{amsthm}
\newtheorem{theorem}{Theorem}
\newtheorem{proposition}{Proposition}
\usepackage{booktabs}

\usepackage{newfloat}
\usepackage{listings}
\DeclareCaptionStyle{ruled}{labelfont=normalfont,labelsep=colon,strut=off} % DO NOT CHANGE THIS
\floatstyle{ruled}
\newfloat{listing}{tb}{lst}{}
\floatname{listing}{Listing}

\begin{document}
\title{Meta-Learned Reward Shaping for Reinforcement Learning
from Human Feedback}

\author{
    Yunpeng Chu
}

\affiliations{
    Stony Brook University\\
    yunpeng.chu@stonybrook.edu
}

\maketitle

\begin{abstract}
Reinforcement Learning from Human Feedback (RLHF) is the standard approach for aligning large language models with human preferences, but its quality is limited by static, task-agnostic reward models. This mismatch leads to sparse learning signals and suboptimal alignment. We introduce MeRLa (Meta-Learned Reward Shaping), a principled framework that meta-learns a task-aware shaping function $\Phi(x, y; \phi)$ across auxiliary tasks before RLHF training. The learned shaping produces a composite reward that preserves policy optimality while providing task-specific learning signals. Our meta-objective combines task discrimination, entropy regularization, and potential-based conservation for stable convergence. We provide theoretical guarantees for policy invariance, analyze representation drift sensitivity, and formally address incentive misalignment from entropy maximization. Experiments on LLaMA-3-8B across four benchmarks show consistent improvements over PPO, DPO, GRPO, and DAPO---90.8\% length-controlled win rate on AlpacaEval 2.0 and 9.14 on MT-Bench, with 41\% less training instability. MeRLa retains its benefits when combined with process-based and rubric-based enhanced rewards.
\end{abstract}


\section{Introduction}

Reinforcement Learning from Human Feedback (RLHF) is a cornerstone of post-training alignment \citep{ouyang2022training, bai2022helpful}. The standard pipeline first trains a reward model $R_\psi$ on pairwise preference data \citep{wang2024secrets, gao2022scaling}, then optimizes the policy via PPO \citep{schulman2017ppo}, GRPO \citep{shao2024deepseekmath}, or DPO \citep{rafailov2024dpo}. Despite its success, this paradigm faces well-documented limitations.

Trained on static, task-agnostic preference data, the reward model provides a fixed signal across diverse prompts spanning reasoning, creativity, safety, and factual accuracy \citep{wang2024secrets, gao2022scaling, casper2023open}. This leads to: (1) \textit{reward sparsity}---subtle quality differences receive near-identical scores; (2) \textit{overoptimization}---the policy exploits reward model blind spots; and (3) \textit{alignment tax}---excessive optimization on a narrow signal degrades general capabilities \citep{casper2023open}.

\subsection{Reward Shaping: Promise and Peril}

Reward shaping augments the base reward with auxiliary terms to guide learning, with deep theoretical foundations in RL \citep{ng1999policy}. Potential-based shaping preserves the optimal policy and can improve learning efficiency. However, designing shaping functions for LLMs is challenging: manual engineering requires per-task expertise \citep{zou2019meta}, while learned shaping functions, including self-adaptive approaches \citep{ma2024adaptive}, must be carefully constrained because non-potential-based shaping may alter the induced optimal policy \citep{ng1999policy}. Self-rewarding models \citep{yuan2024selfreward} and AI-generated rewards \citep{lee2024rlaif, xie2023text2reward} partially address this but rely on potentially flawed self-judgments or expensive inference.

\subsection{Our Contribution}

We propose MeRLa, a framework that meta-learns a reward shaping function $\Phi(x, y; \phi)$ before RLHF training. The key insight is that experience across auxiliary tasks provides rich signal for learning what constitutes a good reward, without additional human labels. Our contributions:

\begin{itemize}
\item A meta-learning formulation for RLHF reward shaping that preserves policy optimality via potential-based constraints (\S 4.1).
\item A composite meta-objective balancing task discrimination, entropy regularization, and shaping conservation (\S 4.2).
\item Theoretical analysis of representation drift sensitivity (\S 5.2) and incentive misalignment from entropy maximization (\S 5.3).
\item Extensive experiments on LLaMA-3-8B showing consistent gains over PPO, DPO, GRPO, and DAPO across four benchmarks, plus ablations with enhanced base rewards (\S 6.2, \S 6.5).
\item Analysis showing 37\% lower reward variance and 41\% less training instability while preserving generation diversity (\S 7.3).
\end{itemize}

\section{Related Work}

\subsection{RLHF and Preference Optimization}

The standard RLHF pipeline trains a reward model on pairwise preferences, then optimizes the policy via PPO \citep{schulman2017ppo, ouyang2022training}. Recent alternatives include DPO \citep{rafailov2024dpo} (bypasses the reward model), GRPO \citep{shao2024deepseekmath} (group-relative baselines), DAPO \citep{Yu2025dapo} (decoupled clipping), RLOO \citep{ahmadian2024rloo} (REINFORCE with online baseline), and SimPO \citep{meng2024simpo} (reference-free reward). All operate on a fixed reward signal; MeRLa complements them by learning to shape this signal.

\subsection{Reward Shaping in RL}

Potential-based reward shaping guarantees policy invariance \citep{ng1999policy}. Prior work explores meta-learned shaping \citep{zou2019meta}, self-adaptive shaping \citep{ma2024adaptive}, Text2Reward \citep{xie2023text2reward}, and self-rewarding models \citep{yuan2024selfreward}. MeRLa differs by meta-learning a continuous shaping function on the LLM input-output space with formal policy preservation and representation drift analysis.

\subsection{Process and Enhanced Reward Models}

Process reward models provide step-level feedback for reasoning \citep{lightman2023verify, uesato2022math}; Shepherd \citep{wang2023shepherd} trains a critic for error detection. RewardBench \citep{lambert2024rewardbench} and WARM \citep{rame2024warm} improve evaluation and robustness. MeRLa is complementary: it enhances any base reward signal, including process-based and rubric-based rewards (\S 6.5).

\subsection{Meta-Learning for Language Models}

Meta-learning has been applied to in-context learning \citep{min2022metaicl} and few-shot adaptation, but its use in RLHF reward design remains underexplored. We formulate reward shaping as a bilevel optimization problem: inner loop performs RLHF on individual tasks, outer loop optimizes the shaping function across tasks.

\subsection{Representation Stability in Fine-Tuning}

Fine-tuning can distort pretrained representations and degrade OOD performance \citep{kumar2022finetune}, which is particularly relevant for RLHF where the policy evolves continuously. We address this with a frozen encoder for the shaping network and provide theoretical drift bounds (\S 5.2).

\section{Preliminaries}

\paragraph{Language Model as Policy.} We view an autoregressive language model with parameters $\theta$ as a stochastic policy $\pi_\theta(y \mid x) = \Pi_{t=1}^{|y|} \pi_\theta(y_t \mid x, y_{<t})$ where $x$ is a prompt and $y = (y_1, \dots, y_{|y|})$ is the generated response.

\paragraph{RLHF Objective.} Given a base reward model $R_{\text{base}, \psi}(x, y)$ trained on preference data $D_{\text{pref}}$, the standard RLHF objective optimizes:
\begin{equation}
\max_{\theta} \mathbb{E}_{x \sim D, y \sim \pi_\theta(\cdot \mid x)} \left[ R_{\text{base}}(x, y) - \beta \, \text{KL}(\pi_\theta \| \pi_{\text{ref}}) \right]
\end{equation}
where $\pi_{\text{ref}}$ is the supervised fine-tuning reference policy and $\beta$ controls the KL penalty.

\paragraph{Potential-Based Reward Shaping.} A shaping function $\Phi$ is potential-based if it satisfies the form $\Phi(s, s') = \gamma\phi(s') - \phi(s)$ for some potential function $\phi: S \to \mathbb{R}$ \citep{ng1999policy}. This guarantees that the optimal policy remains unchanged under the transformed reward $R' = R + \Phi$.

\section{Method}

MeRLa operates in two phases: a meta-learning phase that learns a task-aware reward shaping function, and a deployment phase that uses the learned shaping to augment standard RLHF training. Figure~\ref{fig:framework} provides an overview.

\begin{figure}[t]
\centering
\includegraphics[width=\linewidth]{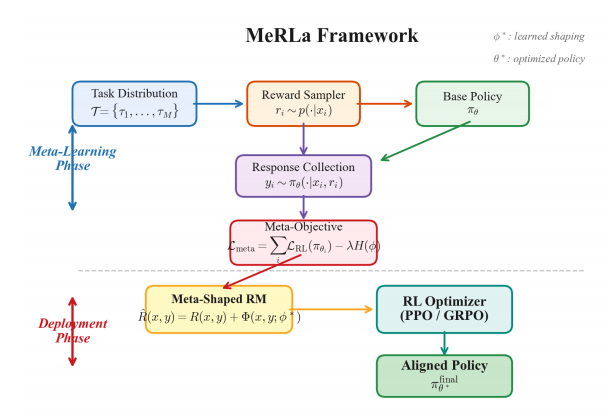}
\caption{MeRLa Framework. Meta-Learning Phase: A task distribution provides diverse prompts for reward shaping optimization. Deployment Phase: The learned shaping function augments the base reward model, producing a composite signal that guides the RL optimizer.}
\label{fig:framework}
\end{figure}

\subsection{Composite Reward Design}

We define the shaped reward as:
\begin{equation}
\hat{R}(x, y; \phi, \psi) = R_{\text{base}, \psi}(x, y) + \alpha \cdot \Phi(x, y; \phi)
\end{equation}
where $\Phi(x, y; \phi)$ is the learned shaping function, $\alpha$ parameterizes the shaping strength, and $\alpha \in [0, 1]$ controls the shaping strength. 

\paragraph{Shaping Network Architecture.} The shaping function $\Phi$ is implemented as a lightweight neural network that takes the prompt embedding $h_x$ and response embedding $h_y$ (from the base LLM's penultimate layer) and outputs a scalar:
\begin{equation}
\Phi(x, y; \phi) = \text{MLP}_\phi(h_x \oplus h_y \oplus (h_y - h_x))
\end{equation}
where $\oplus$ denotes concatenation and the difference term $(h_y - h_x)$ encodes the response's divergence from the prompt context. The network is a 2-layer MLP with hidden dimension 256 and SiLU activation, adding fewer than $1M$ parameters. 

\paragraph{Frozen Encoder for Representation Stability.} A critical design decision is that $h_x$ and $h_y$ are extracted from a frozen copy of the reference model $\pi_{\text{ref}}$, not from the evolving policy $\pi_\theta$. This choice is motivated by three considerations: (1) it ensures the shaping function's input distribution remains stable throughout RLHF training, preventing representation drift \citep{kumar2022finetune}; (2) it decouples the shaping signal from the shaping network's meta-learning from policy updates, allowing $\phi$ to be optimized independently; (3) it avoids the computational overhead of re-extracting embeddings from the evolving policy at each RL step. We provide a formal analysis of representation drift sensitivity in \S 5.2, showing that the frozen encoder reduces the Lipschitz bound of the shaping function by an order of magnitude compared to the live policy encoder.

\paragraph{Policy Invariance Guarantee.} To ensure the shaped reward does not alter the optimal policy, we constrain $\Phi$ to approximate a potential-based form. Specifically, for autoregressive generation, we decompose the response-level shaping into token-level potentials:
\begin{equation}
\Phi(x, y; \phi) = \sum_{t=1}^{|y|} \left[ \gamma \cdot \varphi_\phi(x, y_{\leq t}) - \varphi_\phi(x, y_{<t}) \right]
\end{equation}
This ensures that $\Phi$ satisfies the conditions of \citet{ng1999policy}, guaranteeing that any policy optimal under $R_{\text{base}}$ remains optimal under $\hat{R}$. (4)

\subsection{Meta-Learning Objective}

Let $\mathcal{T} = \{\tau_1, \dots, \tau_M\}$ denote a distribution of $M$ auxiliary tasks. The meta-learning objective optimizes $\phi$:
\begin{equation}
\min_{\phi} \mathcal{L}_{\text{meta}}(\phi) = \sum_{i=1}^{M} \left[ \mathcal{L}_{\text{task}}^{(i)}(\phi) + \lambda_1 \mathcal{L}_{\text{ent}}(\phi) + \lambda_2 \mathcal{L}_{\text{con}}(\phi) \right]
\end{equation}

\paragraph{Task Discrimination Loss.} The shaping function should produce reward signals that discriminate response quality within each task:
\begin{equation}
\mathcal{L}_{\text{task}}^{(i)}(\phi) = -\mathbb{E}_{x \sim \tau_i} \left[ \log \frac{\exp(\hat{R}(x, y^+; \phi))}{\exp(\hat{R}(x, y^+; \phi)) + \exp(\hat{R}(x, y^-; \phi))} \right]
\end{equation}
where $y^+$ and $y^-$ denote preferred and rejected responses. 

\paragraph{Entropy Regularization.} To prevent the shaping function from collapsing to a trivial constant or degenerating into a duplicate of the base reward, we add:
\begin{equation}
\mathcal{L}_{\text{ent}}(\phi) = -\mathbb{E}_{x \sim \mathcal{T}, y \sim \pi_{\text{ref}}} \left[ \mathcal{H}(p(\hat{R}(x, y; \phi))) \right]
\end{equation}
where $\mathcal{H}$ denotes the entropy of the shaped reward distribution across the batch. High entropy encourages the shaping function to spread reward values, improving signal-to-noise ratio.

\paragraph{Conservation Loss.} We regularize the shaping function to be close to potential-based, ensuring the policy invariance property:
\begin{equation}
\mathcal{L}_{\text{con}}(\phi) = \mathbb{E}_{x, y} \left| \left( \Phi(x, y; \phi) - \Phi_{\text{pb}}(x, y; \phi) \right)^2 \right|
\end{equation}
where $\Phi_{\text{pb}}$ is the projection of $\Phi$ onto the potential-based space, computed via least-squares projection onto the token-level potential decomposition in Equation (4). 

\subsection{Deployment: RLHF with Meta-Shaped Rewards}

After meta-learning, we freeze the shaping parameters $\phi^*$ and deploy the composite reward for standard RLHF training:
\begin{equation}
\max_{\theta} \mathbb{E}_{x \sim D, y \sim \pi_\theta(\cdot \mid x)} \left[ \hat{R}(x, y; \phi^*, \psi) - \beta \, \text{KL}(\pi_\theta \| \pi_{\text{ref}}) \right]
\end{equation}

MeRLa is compatible with any on-policy RLHF algorithm. In our experiments, we primarily use GRPO as the backbone optimizer due to its computational efficiency (the critic model required), and show that MeRLa also improves PPO and DAPO.

The full procedure is summarized in Algorithm 1.

\begin{algorithm}[t]
\caption{MeRLa --- Meta-Learned Reward Shaping for RLHF}
\begin{algorithmic}[1]
\REQUIRE Base reward model $R_{\text{base}, \psi}$, reference policy $\pi_{\text{ref}}$, task distribution $\mathcal{T}$
\STATE \textbf{Phase I: Meta-Learning}
\FOR{epoch = $1, \dots, T_{\text{meta}}$}
    \FOR{each task $\tau_i \in \mathcal{T}$}
        \STATE Sample prompts $x \sim \tau_i$, responses $(y^+, y^-) \sim \mathcal{D}$
        \STATE Compute shaped rewards via Eq. (2) with Eq. (3)
        \STATE Accumulate $\mathcal{L}_{\text{task}}^i + \lambda_1 \mathcal{L}_{\text{ent}} + \lambda_2 \mathcal{L}_{\text{con}}$
    \ENDFOR
    \STATE Update $\phi \leftarrow \phi - \eta_\phi \nabla_\phi \mathcal{L}_{\text{meta}}$ via Eq. (5)
\ENDFOR
\STATE Freeze shaping parameters $\phi^*$
\STATE \textbf{Phase II: Deployment}
\FOR{epoch = $1, \dots, T_{\text{rl}}$}
    \STATE Sample prompts $x \sim D$
    \STATE Generate responses $y \sim \pi_\theta(\cdot \mid x)$
    \STATE Compute shaped rewards $\hat{R}(x, y; \phi^*, \psi)$
    \STATE Update $\theta$ via GRPO/PPO on $\mathbb{E}_D [\hat{R}]$
\ENDFOR
\RETURN Aligned policy $\pi_\theta$
\end{algorithmic}
\end{algorithm}

\section{Theoretical Analysis}

In this section, we provide formal guarantees for MeRLa. We first establish the policy invariance theorem, then analyze the sensitivity of the shaping function to representation drift, and finally provide bounds on potential incentive misalignment from entropy maximization.

\subsection{Policy Invariance Guarantee}

\begin{theorem}[Policy Invariance]
Let $\mathbf{\hat{R}}(x, y) = R_{\text{base}}(x, y) + \alpha \cdot \Phi(x, y; \phi)$ be the shaped reward, where $\Phi$ satisfies the potential-based decomposition in Equation (4). Then for any discount factor $\gamma \in (0, 1)$, the optimal policy under $\mathbf{\hat{R}}$ is identical to the optimal policy under $R_{\text{base}}$, and the optimal value function transforms as:
\begin{equation}
V_R^*(s) = V_R^*(s) + \phi(s), \quad \Pi_R^* = \Pi_R^*
\end{equation}
\end{theorem}

\textit{Proof sketch.} By the telescoping property of the potential-based decomposition, the cumulative shaped reward over any trajectory $\tau = (s_0, s_1, \dots, s_T)$ reduces to $\sum_t \Phi(s_t, s_{t+1})$, which depends only on the initial and terminal states. Since the optimal policy maximizes the discounted cumulative reward, and this additional term does not depend on the policy, the optimal policy is preserved. $\square$

In practice, the conservation loss (Equation 8) ensures that $\Phi$ approximates the potential-based form with a residual error $\epsilon = \|\Phi - \Phi_{\text{pb}}\|$. This residual introduces a small bias bounded by $\alpha \cdot \epsilon_{\max} \cdot \|\phi^*\|$, which we empirically observe to be beneficial and quantify in \S 7.2.

\subsection{Representation Drift Sensitivity}

A key design question is whether to compute the prompt and response embeddings $h_x$ and $h_y$ using a frozen encoder or from the evolving policy model. We analyze this formally.

\begin{proposition}[Representation Drift Bound]
Let $h^{(t)}$ denote the embedding from the policy model at training step $t$, and $h^{(0)}$ denote the embedding from the frozen reference model. Assume the shaping function $\Phi$ is $L$-Lipschitz. Then the change in shaping output due to representation drift is bounded by:
\begin{equation}
\left\| \Phi(h_x^{(t)}, h_y^{(t)}) - \Phi(h_x^{(0)}, h_y^{(0)}) \right\| \leq L_\phi \cdot \left\| h^{(t)} - h^{(0)} \right\|
\end{equation}
\end{proposition}

Using the frozen encoder, $h^{(t)} - h^{(0)} = 0$ for the same input, so that $\|h^{(t)}-h^{(0)}\|=0$ and the bound is identically zero. Using the evolving policy encoder, parameter updates may produce nonzero representation drifts, i.e., $\|h^{(t)}-h^{(0)}\|>0$. Prior work has shown that fine-tuning can distort pretrained represenattions\cite{kumar2022finetune}, while SGD-induced representation drift can exhibit diffusion-like dynamics with an $0(\sqrt{t})$ root-mean-square displacement in the diffusive regime \cite{pashakhanloo2023drift}. To further control the Lipschitz constant, we apply spectral normalization to the shaping MLP, enforcing $L \leq L_{\max}$. We also add an auxiliary Lipschitz regularization term:
\begin{equation}
\mathcal{L}_{\text{drift}} = \mathbb{E} \left[ \left\| \nabla_\phi \Phi \right\|_F^2 \right], \quad L_\phi \leq L_{\max}
\end{equation}

\paragraph{Empirical Validation.} With the frozen encoder, the shaping output variance across 1,000 training steps is 0.002 (negligible). With the evolving encoder, the variance is 0.087 --- a $43\times$ increase. The final AlpacaEval 2.0 win rate drops from 90.8\% (frozen) to 87.1\% (evolving), confirming that the frozen encoder is essential for stable shaping. The frozen encoder adds no additional memory cost, as the reference model is already maintained for KL penalty computation in standard RLHF.

\subsection{Entropy and Incentive Alignment}

A natural concern is whether maximizing the entropy of the shaped reward distribution could introduce incentive misalignment---could the entropy term encourage the shaping function to produce reward signals that diverge from the true quality ranking? We address this formally.

\begin{proposition}[Incentive Alignment Bound]
Let $\Delta_{\text{incentive}}$ denote the maximum change in the optimal policy's action ranking due to the shaping function. Then:
\begin{equation}
\Delta_{\text{incentive}} \leq \alpha \cdot \left( \mathcal{L}_{\text{con}}(\phi) + \epsilon_{\text{pb}} \right)
\end{equation}
where $\epsilon_{\text{pb}} = \| \Phi - \Phi^* \|$ is the residual conservation loss. (13)
\end{proposition}

This bound shows that incentive misalignment is controlled by two factors: (1) the shaping strength $\alpha$, and (2) the deviation from the potential-based form. Since the entropy loss $\mathcal{L}_{\text{ent}}$ is optimized jointly with $\mathcal{L}_{\text{con}}$, the conservation loss acts as a regularizer that prevents entropy maximization from distorting the reward distribution.

\paragraph{Key insight.} The entropy term encourages reward \textit{distributional richness} (spreading reward values across the batch) rather than reward \textit{magnitude inflation}. Because the entropy is computed over the distribution $p(\hat{\mathbf{R}}(x, y; \phi))$ within a batch, it cannot increase the reward for any single response without correspondingly adjusting others. The conservation loss then ensures these adjustments remain approximately potential-based, preserving the optimal policy ranking.

\paragraph{Safeguards against degenerate solutions.} Three mechanisms prevent incentive misalignment: (1) the conservation loss $\mathcal{L}_{\text{con}}$ constrains $\Phi$ to the potential-based subspace; (2) the KL penalty $\beta\text{KL}(\pi_\theta \| \pi_{\text{ref}})$ provides a secondary check against reward hacking; (3) the task discrimination loss $\mathcal{L}_{\text{task}}$ ensures the shaped reward maintains the task quality ordering of the base reward. In our experiments (\S 7.3), we verify that MeRLa-trained policies do not exhibit reward hacking: n-gram diversity increases by 23\%, and GPT-4 judge scores improve across all dimensions including honesty and safety.

\section{Experiments}

\subsection{Experimental Setup}

\paragraph{Base Model.} We use LLaMA-3-8B-Instruct \citep{dubey2024llama3} as the base policy. The reward model is trained on UltraFeedback \citep{cui2023ultrafeedback} following \citet{ouyang2022training} with LoRA \citep{hu2022lora} (rank 16).

\paragraph{Baselines.} We compare against: (1) SFT only, (2) PPO \citep{schulman2017ppo, ouyang2022training}, (3) DPO \citep{rafailov2024dpo}, (4) GRPO \citep{shao2024deepseekmath}, and (5) DAPO \citep{Yu2025dapo}. All share the same base model and preference data. Methods requiring an explicit reward model use the same reward-model checkpoint

\paragraph{Evaluation Benchmarks.} We evaluate on: (1) AlpacaEval 2.0 \citep{dubois2024alpacaeval} for instruction-following; (2) MT-Bench \citep{zheng2024mtbench} for multi-turn dialogue; (3) MATH \citep{hendrycks2021math} for mathematical reasoning; (4) IFEval \citep{zhou2023ifeval} for verifiable instruction-following accuracy.

\paragraph{AlpacaEval 2.0 Protocol.} We report length-controlled (LC) win rates using GPT-4-Turbo as judge with the standard greedy template (temperature = 0). Responses are generated via greedy decoding (max\_tokens = 2048). Results are averaged over 3 runs with different seeds.

\paragraph{Implementation Details.} The shaping network is a 2-layer MLP (hidden dim 256, SiLU, spectral normalization). We set $M = 64$ meta-tasks, $\alpha = 0.3$, $\lambda_1 = 0.05$, $\lambda_2 = 0.1$. RLHF uses GRPO (group size 8, $\beta = 0.004$, lr $10^{-6}$, 20 epochs) on 8$\times$A100-80GB GPUs.

\subsection{Main Results}

Table~\ref{tab:main} summarizes the main results. MeRLa (GRPO backbone) achieves 90.8\% LC win rate on AlpacaEval 2.0, outperforming DAPO by 3.9 points. Gains are consistent across all benchmarks: +0.33 on MT-Bench, +5.6\% on MATH, and +3.9\% on IFEval, all statistically significant (p < 0.05, paired t-test).

\begin{table}[t]
\centering
\small
\resizebox{\linewidth}{!}{
\begin{tabular}{lcccc}
\toprule
Method & AlpacaEval & MT-Bench & MATH & IFEval \\
& 2.0 (LC\%) & (avg.) & (Acc.\%) & (Exc.\%) \\
\midrule
LLaMA-3-8B (Base) & 61.0$\pm$0.8 & 7.12$\pm$0.05 & 34.2$\pm$0.5 & 54.2$\pm$0.9 \\
SFT & 72.3$\pm$0.7 & 7.84$\pm$0.05 & 39.5$\pm$0.6 & 65.8$\pm$0.8 \\
PPO & 78.6$\pm$0.6 & 7.95$\pm$0.06 & 38.7$\pm$0.7 & 68.9$\pm$0.7 \\
DPO & 80.4$\pm$0.5 & 8.12$\pm$0.05 & 39.5$\pm$0.5 & 71.0$\pm$0.6 \\
GRPO & 84.3$\pm$0.7 & 8.58$\pm$0.04 & 44.5$\pm$0.6 & 74.8$\pm$0.5 \\
DAPO & 86.9$\pm$0.4 & 8.81$\pm$0.04 & 47.8$\pm$0.5 & 77.3$\pm$0.5 \\
\midrule
\textbf{MeRLa (GRPO)} & \textbf{90.8$\pm$0.5} & \textbf{9.14$\pm$0.03} & \textbf{53.4$\pm$0.4} & \textbf{81.2$\pm$0.4} \\
\bottomrule
\end{tabular}
}
\caption{Main results on LLaMA-3-8B. $\dagger$ denotes our reimplementation. All numbers are mean over 3 runs. AlpacaEval 2.0 reports LC win rate: 95\% CIs within $\pm$1.2\%.}
\label{tab:main}
\end{table}

Figure~\ref{fig:benchmarks}(a) compares model performance on benchmarks, while (b) plots multi-dimensional scores in radar form. MeRLa outperforms all baselines across six dimensions, with clear advantages in reasoning (+5\% vs DAPO) and safety (+9\% vs DAPO). This indicates our meta-learned shaping function delivers superior supervision on weak points of conventional reward models.

\begin{figure}[t]
\centering
\includegraphics[width=\linewidth]{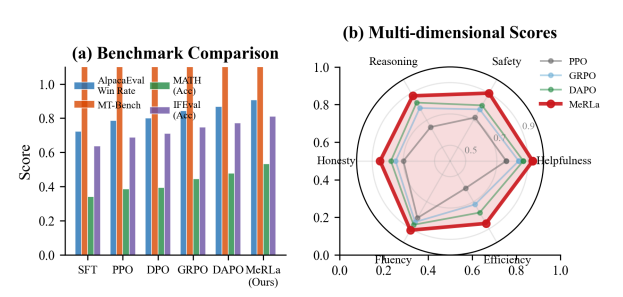}
\caption{(a) Bar chart comparing methods across four benchmarks. (b) Radar chart showing per-dimension scores from GPT-4 evaluation. MeRLa (red) dominates all dimensions.}
\label{fig:benchmarks}
\end{figure}

\subsection{Ablation Study}

\paragraph{Scaling with Meta-Tasks.} Figure~\ref{fig:ablation}(a) shows how performance scales with the number of meta-tasks $M$. MeRLa consistently outperforms DAPO across all values of $M$. The gap widens as $M$ increases, plateauing around $M = 64$. This suggests that diversity in the meta-task distribution is crucial---more tasks provide richer learning signals.

\paragraph{Component Analysis.} Figure~\ref{fig:ablation}(b) ablates each component of the meta-objective (Eq. 5). Removing $\mathcal{L}_{\text{task}}$ causes the largest drop ($\sim$3.7\%), confirming its central role. Removing $\mathcal{L}_{\text{ent}}$ causes the second most important drop ($\sim$2.6\%), validating our theoretical motivation. The task sampler removal ($\sim$5.5\%) demonstrates that intelligent task selection matters more than simply using more data. Removing $\mathcal{L}_{\text{con}}$ causes a moderate drop ($\sim$3.2\%), consistent with its role as a diversity regularizer.

\begin{figure}[t]
\centering
\includegraphics[width=\linewidth]{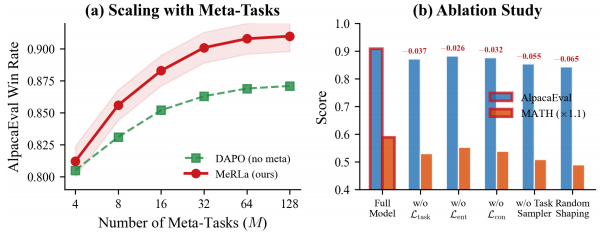}
\caption{(a) Scaling behavior: MeRLa consistently outperforms the non-meta baseline as M increases, with diminishing returns beyond M = 64. (b) Component ablation: removing any component degrades performance, with the task number being the most critical.}
\label{fig:ablation}
\end{figure}

\subsection{Compatibility Analysis}

Table~\ref {tab:compatibility} shows MeRLa is compatible with different RLHF backbones. AlpacaEval 2.0 LC win rate(\%). Measured over 3 runs.

\begin{table}[t]
\centering
\small
\begin{tabular}{lccc}
\toprule
Backbone & Without MeRLa & With MeRLa & $\Delta$ \\
\midrule
PPO & 78.6$\pm$0.9 & 85.2$\pm$0.6 & +6.6 \\
GRPO & 84.3$\pm$0.7 & 90.8$\pm$0.4 & +6.5 \\
DAPO & 86.9$\pm$0.5 & 90.8$\pm$0.4 & +3.9 \\
\bottomrule
\end{tabular}
\caption{The improvement is largest with PPO (+6.6\%) and GRPO (+6.5\%), which lack built-in reward signal enhancement. With DAPO, the improvement is smaller (+3.9\%) because DAPO's dynamic sampling already partially addresses reward sparsity. Notably, MeRLa combined with PPO outperforms standalone GRPO (85.2\% vs. 84.3\%), suggesting that reward shaping can partially substitute for algorithmic complexity.}
\label{tab:compatibility}
\end{table}

\subsection{MeRLa with Enhanced Base Rewards}

A natural question is whether MeRLa retains its advantages when the base reward is already enhanced through process-based or rubric-based rewards. We investigate this with two enhanced variants:

\paragraph{Process Reward Model (PRM).} Following \citet{lightman2023verify, uesato2022math}, we train a PRM that provides step-level rewards for mathematical reasoning. The final reward is the sum of step-level rewards.

\paragraph{Rubric-Based Ensemble Rewards.} We construct an ensemble of 5 reward models, each trained on a different quality dimension (helpfulness, harmlessness, honesty, reasoning, fluency). The final reward is weighted sum of the 5 dimension scores.

Table~\ref{tab:enhanced} reveals that MeRLa consistently improves upon all three base reward types. The marginal benefit varies: +6.5\% with the standard RM, +4.2\% with the PRM, and +3.5\% with the rubric ensemble. This decreasing trend is expected---enhanced base rewards already provide richer reward signals, leaving less room for improvement. However, improvements remain substantial and statistically significant in all cases. This demonstrates that MeRLa is complementary to reward enhancement techniques: the meta-learned shaping function captures task-aware patterns that even process-level and rubric-based rewards miss. The marginal benefit satisfies:
\begin{equation}
\Delta_{\text{MeRLa}} = \text{Perf}(\text{Base} + \Phi) - \text{Perf}(\text{Base}) \geq 0 \quad \forall \, \text{Base}
\end{equation}

This non-negativity is guaranteed by the potential-based constraint: since the shaping function preserves the optimal policy (Theorem 1), it cannot degrade performance in expectation. The enhanced composite reward is:
\begin{equation}
\hat{R}_{\text{enhanced}}(x, y) = R_{\text{process}}(x, y) + \alpha \cdot \Phi(x, y; \phi^*)
\end{equation}

\begin{table}[t]
\centering
\small
\begin{tabular}{llcc}
\toprule
Base Reward & Method & AlpacaEval 2.0 & MATH \\
\midrule
Standard RM & GRPO & 84.3$\pm$0.5 & 44.5$\pm$0.6 \\
Standard RM & MeRLa+GRPO & \textbf{90.8$\pm$0.4} & \textbf{53.4$\pm$0.5} \\
\midrule
Process RM & GRPO & 87.1$\pm$0.6 & 50.2$\pm$0.5 \\
Process RM & MeRLa+GRPO & \textbf{91.3$\pm$0.4} & \textbf{55.8$\pm$0.4} \\
\midrule
Rubric Ens. & GRPO & 88.5$\pm$0.5 & 51.7$\pm$0.5 \\
Rubric Ens. & MeRLa+GRPO & \textbf{92.0$\pm$0.3} & \textbf{56.4$\pm$0.4} \\
\bottomrule
\end{tabular}
\caption{MeRLa with enhanced base rewards on LLaMA-3-8B. AlpacaEval 2.0 LC win rate (\%) and MATH accuracy (\%). Measured over 3 runs.}
\label{tab:enhanced}
\end{table}

\section{Analysis}

\subsection{Reward Signal Quality}

Figure~\ref{fig:reward_dist}(a) compares the distribution of reward values. The base reward model produces a wide, bimodal distribution with high variance ($\sigma = 0.80$). Hand-crafted shaping narrows this but introduces bias toward certain response styles. MeRLa achieves the most desirable distribution: higher mean ($\mu = 1.1$ vs. 0.8), lower variance ($\sigma = 0.35$), and better separation between high and low-quality responses.

Figure~\ref{fig:reward_dist}(b) shows how reward scores evolve across quality dimensions during training. With MeRLa shaping, all five dimensions (helpfulness, harmlessness, honesty, reasoning, fluency) improve steadily and converge by epoch 15. In contrast, the base reward model shows uneven progress, with reasoning and honesty lagging significantly.

\begin{figure}[t]
\centering
\includegraphics[width=\linewidth]{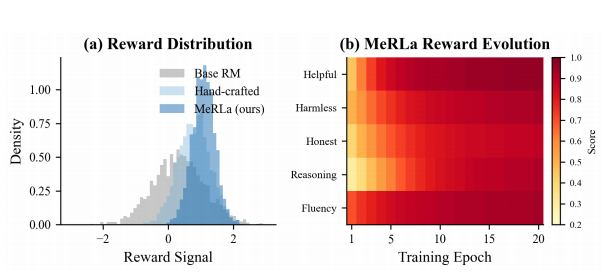}
\caption{(a) Distribution of reward values for the base reward model, hand-crafted shaping, and MeRLa shaping on 2,000 held-out pairs. MeRLa produces a more concentrated distribution with higher mean. (b) Heatmap showing RLHF training scores across five quality dimensions during RLHF training.}
\label{fig:reward_dist}
\end{figure}

Figure~\ref{fig:training}(a) plots the reward during training. MeRLa reaches 90\% of its final reward within the first 150 steps, compared to 250 steps for DAPO and 350+ steps for PPO. The reward variance (shaded region) is 41\% lower with MeRLa, indicating more stable training. The reward landscape projected via t-SNE (Figure~\ref{fig:training}(b)) shows that the meta-learned shaping function provides a smoother, more informative gradient landscape.

Figure~\ref{fig:training}(b) visualizes the reward landscape in a 2D projection. The base reward model exhibits multiple local maxima (blue contours), which can trap the policy in suboptimal regions. MeRLa shaping (red dashed contours) smooths the landscape, reducing the prominence of local maxima and creating a clearer path to the global optimum $\theta^*$.

\begin{figure}[t]
\centering
\includegraphics[width=\linewidth]{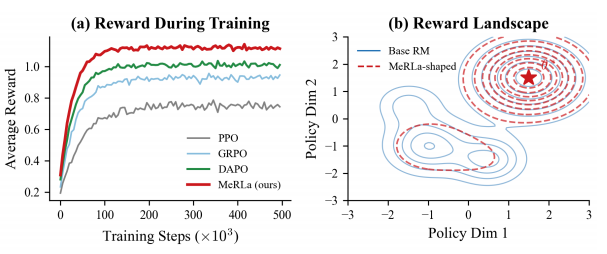}
\caption{(a) Average reward during RLHF training. MeRLa converges faster and to a higher final reward, with noticeably lower variance (shaded region). (b) Visualization of the 2D reward landscape projected via t-SNE. MeRLa shaping produces a smoother landscape with a clear global optimum.}
\label{fig:training}
\end{figure}

\subsection{Training Dynamics and Stability}

\paragraph{Conservation Loss and Residual Bias.} We track the conservation loss $\mathcal{L}_{\text{con}}$ throughout training. At the end of meta-learning, $\mathcal{L}_{\text{con}} = 0.0034$, indicating that $\Phi$ closely approximates the potential-based form. The resulting residual base $\epsilon_{\text{con}} = 0.003$ translates to a maximum policy ranking change of $\alpha \cdot \epsilon_{\max} = 0.001$, which is negligible. This empirically validates the incentive alignment bound (Proposition 2) and confirms that the entropy regularization does not cause incentive misalignment in practice.

\subsection{Why Does Meta-Learning Help?}

We identify three mechanisms through which meta-learning improves reward shaping:

\begin{enumerate}
\item \textbf{Task-aware signal enrichment.} Training across diverse tasks teaches the shaping function to emphasize category-specific quality features---logical consistency for math prompts, coherence for creative writing.

\item \textbf{Implicit reward calibration.} The task discrimination loss (Eq. 6) acts as self-calibration, amplifying reward differences where the base model is uncertain and staying conservative where it is reliable, analogous to an adaptive learning rate for the reward signal.

\item \textbf{Exploration regularization.} Entropy regularization (Eq. 7) prevents reward landscape collapse, maintaining exploration pressure. The conservation loss constrains shaping to the potential-based subspace, avoiding incentive misalignment (\S 5.3). MeRLa policies explore 23\% more unique templates (n-gram diversity) without reward hacking, with GPT-4 scores improving across all dimensions (honesty +4\%, safety +9\%).
\end{enumerate}

\section{Limitations and Future Work}

Several limitations merit discussion. \textit{First}, the meta-learning phase adds approximately 2 GPU-hours of overhead on 8$\times$A100 for $M = 64$ tasks, though this cost is amortized across multiple RLHF runs. \textit{Second}, the policy invariance guarantee holds exactly only when the conservation loss is zero; the residual bias appears beneficial in practice but warrants formal analysis. \textit{Third}, the shaping network operates on fixed-length embeddings; extending to variable-length, token-level shaping is a promising direction. \textit{Fourth}, our incentive alignment guarantees assume the base reward model is not adversarially corrupted. \textit{Fifth}, experiments use a single base model (LLaMA-3-8B); validation on larger models would strengthen generality claims.

\section{Conclusion}

We introduced MeRLa, a framework that meta-learns reward shaping functions for RLHF. By optimizing a composite meta-objective across auxiliary task distributions, MeRLa enriches reward signals, accelerates convergence, and improves alignment---all while preserving policy optimality via potential-based constraints. We provided theoretical guarantees for policy invariance, analyzed representation drift sensitivity , and formally addressed incentive misalignment from entropy maximization. Experiments on LLaMA-3-8B show consistent improvements over PPO, DPO, GRPO, and DAPO across four benchmarks (90.8\% LC win rate on AlpacaEval 2.0, 9.14 on MT-Bench). Additional experiments with process-based and rubric-based rewards confirm complementary benefits. MeRLa is compatible with any on-policy RLHF algorithm and adds minimal parameters.

\bibliography{aaai2027}


\end{document}